\documentclass[sigconf]{acmart}
\usepackage[utf8]{inputenc} 
\usepackage[T1]{fontenc}    
\usepackage{hyperref}       
\usepackage{url}            
\usepackage{booktabs}       
\usepackage{amsfonts}       
\usepackage{nicefrac}       
\usepackage{microtype}      
\usepackage{xcolor}         
\usepackage{makecell}
\usepackage{multirow}
\usepackage{subfigure}
\usepackage{adjustbox}
\usepackage{pifont}

\usepackage{amsthm}
\usepackage[utf8]{inputenc}
\newtheorem{theorem}{Theorem}

\newtheorem{definition}{Definition}

\newcommand*{\argmin}{\operatornamewithlimits{argmin}\limits}

\newcommand{\norm}[1]{\left\lVert#1\right\rVert}
\newcommand{\R}{\mathbb{R}}
\usepackage{amsmath,amsfonts,amsthm,amsopn,bm}

\usepackage{amssymb}
\usepackage{algorithm}
\usepackage{algorithmic}

\usepackage{enumitem}
\setlist{topsep=0pt, leftmargin=*}

\AtBeginDocument{%
  }

\copyrightyear{2023} 
\acmYear{2023} 
\setcopyright{acmlicensed}\acmConference[KDD '23]{Proceedings of the 29th
ACM SIGKDD Conference on Knowledge Discovery and Data Mining}{August
6--10, 2023}{Long Beach, CA, USA}
\acmBooktitle{Proceedings of the 29th ACM SIGKDD Conference on Knowledge
Discovery and Data Mining (KDD '23), August 6--10, 2023, Long Beach, CA,
USA}
\acmPrice{15.00}
\acmDOI{10.1145/3580305.3599466}
\acmISBN{979-8-4007-0103-0/23/08}

\begin{document}

\title{Physics-Guided Discovery of Highly Nonlinear Parametric \\Partial Differential Equations}

\author{Yingtao Luo}
\authornote{The two authors have equal contribution to this work.}
\authornotemark[3]
\affiliation{
  \institution{Carnegie Mellon University}
  \city{Pittsburgh}
  \country{USA}
  }
\email{yingtaol@andrew.cmu.edu}

\author{Qiang Liu}
\authornotemark[1]
\authornote{To whom correspondence should be addressed.}
\authornotemark[3]
\affiliation{
  \institution{CRIPAC, MAIS,\\Institute of Automation, Chinese Academy of Sciences}
  \city{Beijing}
  \country{China}
  }
\email{qiang.liu@nlpr.ia.ac.cn}

\author{Yuntian Chen}
\authornotemark[2]
\authornote{This work was done when these authors were visiting or working at RealAI.}
\affiliation{
  \institution{Ningbo Institute of Digital Twin, Eastern Institute of Technology}
  \city{Ningbo}
  \country{China}
  }
\email{ychen@eias.ac.cn}

\author{Wenbo Hu}
\affiliation{
  \institution{Hefei University of Technology}
  \city{Hefei}
  \country{China}
  }
\email{wenbohu@hfut.edu.cn}

\author{Tian Tian}
\affiliation{
  \institution{RealAI}
  \city{Beijing}
  \country{China}
  }
\email{tian.tian@realai.ai}

\author{Jun Zhu}
\affiliation{
  \institution{Dept. of Comp. Sci. and Tech., Institute for AI, THBI Lab, BNRist Center, Tsinghua-Bosch Joint ML Center, Tsinghua University}
  \city{Beijing}
  \country{China}
  }
\email{dcszj@tsinghua.edu.cn}

\renewcommand{\shortauthors}{Trovato et al.}

\begin{abstract}
Partial differential equations (PDEs) that fit scientific data can represent physical laws with explainable mechanisms for various mathematically-oriented subjects, such as physics and finance.
The data-driven discovery of PDEs from scientific data thrives as a new attempt to model complex phenomena in nature, but the effectiveness of current practice is typically limited by the scarcity of data and the complexity of phenomena.
Especially, the discovery of PDEs with highly nonlinear coefficients from low-quality data remains largely under-addressed.
To deal with this challenge, we propose a novel physics-guided learning method, which can not only encode observation knowledge such as initial and boundary conditions but also incorporate the basic physical principles and laws to guide the model optimization.
We theoretically show that our proposed method strictly reduces the coefficient estimation error of existing baselines, and is also robust against noise.
Extensive experiments show that the proposed method is more robust against data noise, and can reduce the estimation error by a large margin.
Moreover, all the PDEs in the experiments are correctly discovered, and for the first time we are able to discover three-dimensional PDEs with highly nonlinear coefficients.
\end{abstract}




\maketitle

\section{Introduction}
\label{sec:introduction}
Partial differential equations (PDEs) are ubiquitous in many areas, such as physics, engineering, and finance. PDEs are highly concise and understandable expressions of physical mechanisms, which are essential for deepening our understanding of the world and predicting future responses.
The discovery of some typical PDEs is considered as milestones of scientific advances, such as the Navier-Stokes equations and Kuramoto–Sivashinsky equations in fluid dynamics, the Maxwell’s equations and Helmholtz equations in electrodynamics, and the Schr\"odinger’s equations in quantum mechanics. Nevertheless, there are still a lot of unknown complex phenomena in modern science such as the micro-scale seepage and turbulence governing equations that await PDEs for description. 

\begin{figure*}[t]
\centering
\includegraphics[width=0.75\linewidth]{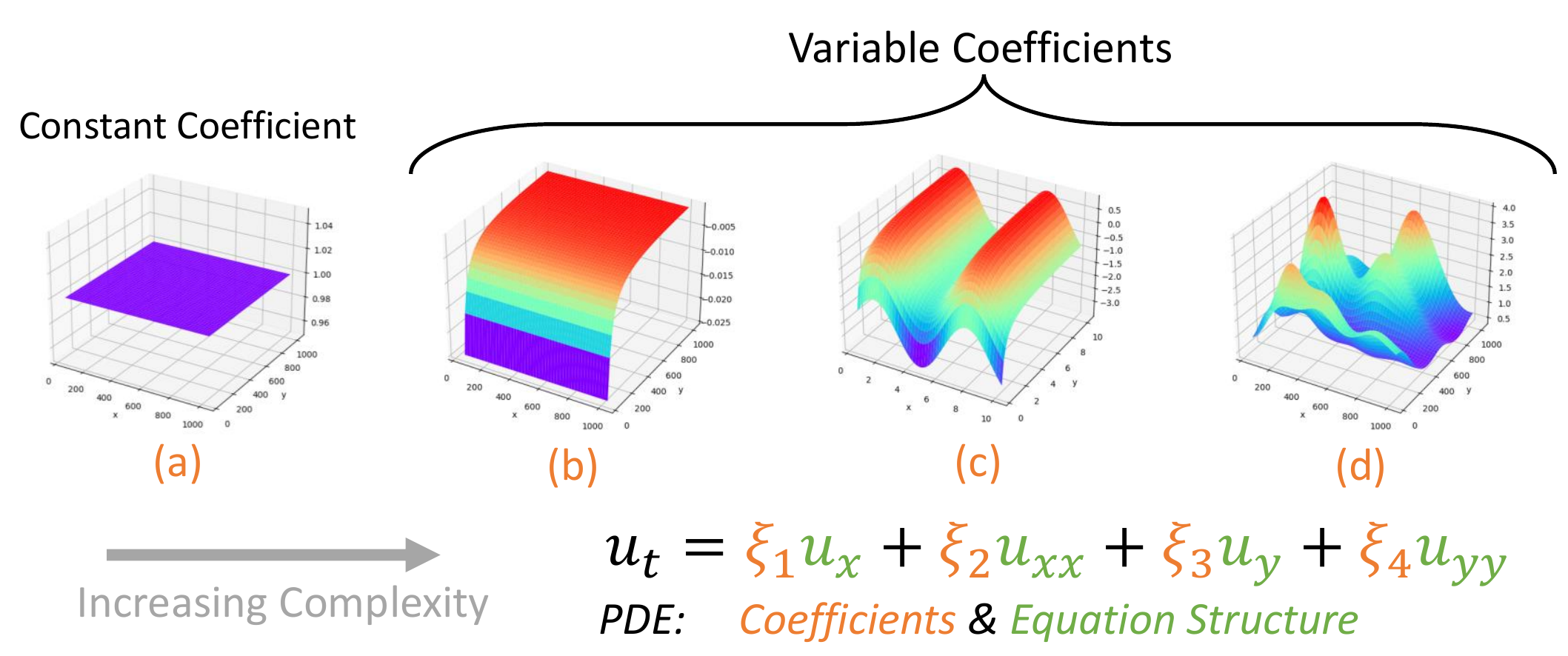}
\caption{Schematic diagram of PDE coefficients. From left to right, the complexity of coefficient fields increases: (a) a constant value $1$, (b) $-1/x$, (c) $-1/x + \sin y$, (d) complex coefficient field described by the Karhunen-Loève expansion of plenty of smooth basis functions \cite{zhang2004efficient, huang2001convergence}.}
\label{fig:coefficient}
\end{figure*}

Traditionally, PDEs are mainly discovered by: 1) mathematical derivation based on physical laws or principles (e.g., conservation laws and minimum energy principles); and 2) analysis of experimental observations. 
With the increasing dimensions and nonlinearity of the physical problems to be solved, the PDE discovery is becoming increasingly challenging, which motivates people to take advantage of machine learning methods.
Pioneering works \cite{bongard2007automated, schmidt2009distilling} use symbolic regression to reveal the differential equations that govern nonlinear dynamical systems without using any prior knowledge. 
More recently, the representative SINDy\cite{brunton2016discovering} and STRidge \cite{rudy2017data} algorithms are proposed assuming that the dynamical systems are essentially controlled by only few dominant terms. Through sparse regression, feature selection from candidate terms is performed to estimate the PDE model \cite{rudy2019data, xu2020deep}. 
Further attempts make use of observation knowledge such as boundary conditions of PDEs \cite{raissi2019physics, rao2021discovering, chen2021physics} and low-rank property of scientific data \cite{li2020robust}, which greatly reduce the quality of data needed for PDE discovery.

Although the aforementioned works show promise in discovering PDEs with constant coefficients (PDEs-CC) as shown in Fig. \ref{fig:coefficient}(a) and some simple instances of parametric PDEs (PDEs with variable coefficients, PDEs-VC) as shown in Fig. \ref{fig:coefficient}(b)-(c), they do not yet suffice to discover more complex PDEs (e.g., PDEs with highly nonlinear coefficients) from scarce and noise data. An example of highly nonlinear coefficients is the permeability random field \cite{zhang2004efficient, huang2001convergence} shown in Fig. \ref{fig:coefficient}(d) for the spatial derivative terms in the PDEs of the seepage problems.
Moreover, the PDEs obtained purely based on data-driven methods can only minimize the estimation error, but these methods may not consider the satisfaction of physical principles, such as the conservation of energy, momentum, etc.

To address these challenges, we rethink how the traditional PDE discovery works.
Based on physical principles, scientists ensure that a newly discovered PDE aligns with the physical world. 
For example, the Navier-Stokes (NS) equation originates from the conservation of momentum, thus each term can relate to a certain physical meaning like convective accumulation or viscous momentum. 
Inspired by this, we propose a HIghly Nonlinear parametric PDE discovery (HIN-PDE) framework.
HIN-PDE is a physics-guided learning framework that not only uses observation knowledge such as initial conditions and assumed terms for certain problems but also uses basic physical principles that are universal in nature as learning constraints to guide model optimization. 
Under this framework, a spatial kernel sparse regression model is proposed considering the principles of smoothness (as a first principle in PDEs) and conservation to impose smoothing of adjacent spatial coefficients for discovering PDEs with highly nonlinear coefficients. 

Furthermore, we have conducted extensive experiments.
Experimental results demonstrate that the proposed method can increase the overall accuracy of PDE estimation and the model robustness by a large margin.
In particular, we consider the discovery of PDEs of different structure complexities with comparisons to baselines.
Our method can discover the PDE structures of all instances that align well with the existing physical principles, while other baselines yield false equation structures for some complex PDEs with excessively high estimation errors.

To summarize, the main contributions of this paper are listed as follows:
\begin{itemize}
\item We propose a novel physics-guided framework for discovering PDEs from sparse and noisy data, which not only encodes observation knowledge but also incorporates physical principles to decrease errors and alleviate data quality issues. We propose a spatial kernel sparse regression model that considers conservation and differentiation principles. It presents excellent robustness in spite of noise compared to previous baselines, and can apply to sparse data in continuous spaces without fixed grids.
\item We theoretically prove that our proposed method strictly reduces the coefficient estimation error of existing baselines, and is also more robust against noise.
\item We report experiments on representative datasets of nonlinear systems with comparison to strong baselines. Compared with the baselines, the results show that our method has much lower coefficient estimation errors, and is the only one that can discover all the test PDEs with highly nonlinear variable coefficients even when the data is noisy.
\end{itemize}

\section{Related Work}

In this section, we introduce some related work.

\subsection{Dynamical System Modeling}

Machine learning is widely leveraged to predict the future response of desired physical fields from data \cite{morton2018deep, li2019learning}. As an alternative, scientists can also obtain the future response by solving a partial differential equation (PDE) that describes the dynamical system. 
Early pioneering works \cite{Lee1990NeuralAF, Lagaris1998ArtificialNN} using neural networks to simulate dynamical systems can date back to three decades ago.
More recent machine learning algorithms \cite{guo2016convolutional, kutz2017deep, amos2017optnet, sirignano2018dgm, li2021physical, wang2020towards} can be mainly divided into two branches: the mesh-based discrete learning and the meshfree continuous learning of simulation. 
Within the meshfree learning branches, pure data-driven approaches \cite{Han2018SolvingHP, long2018pde} are mainly based on high-quality data and physics-informed approaches \cite{raissi2019physics, rao2021physics, karniadakis2021physics, chen2021theory} use physics knowledge to enhance models to adapt to noisier and sparser data. 
Recent studies on neural operators \cite{li2020fourier, lu2021learning} also use neural networks to learn the meshfree and infinite-dimensional mapping for dynamical systems. 
Within the mesh-based learning branches, convolutional networks are widely adopted \cite{zhu2018bayesian, zhu2019physics} to simulate PDEs for spatiotemporal systems \cite{bar2019learning, geneva2020modeling, kochkov2021machine, gao2021phygeonet}. 
The geometry-adaptive learning of nonlinear PDEs with arbitrary domains \cite{belbute2020combining, sanchez2020learning, gao2022physics} and the particle-based dynamical system modeling \cite{li2018learning, ummenhofer2019lagrangian} by graph neural networks rises as a promising direction. 
Moreover, deep learning also renders giving symbolic representation of solutions to PDEs \cite{lample2019deep} possible and demonstrate higher accuracy \cite{magill2018neural, um2020solver, li2020multipole}. 

\subsection{Data-driven Discovery} 

Early trials for equation discovery in the last century \cite{dzeroski1995discovering} uses inductive logic programming to find the natural laws. Two research streams have been proposed to search the governing equations. The first stream aims at identifying a symbolic model \cite{cranmer2020discovering} that describes the dynamical systems from data, which uses symbolic regression \cite{bongard2007automated, schmidt2009distilling} and symbolic neural networks \cite{sahoo2018learning, kim2020integration} to discover functions by comparing differentiation of the experimental data with analytic derivatives of candidate function. The second stream is mainly to incorporate prior knowledge \cite{chen2022integration, chen2021physics, rao2021discovering} and perform sparse regressions \cite{rudy2017data, schaeffer2017learning, brunton2016discovering, raissi2018hidden, bar2019learning, luo2022learning} to discover PDEs by selecting term candidates. Evolutionary algorithms \cite{xu2020dlga, chen2022symbolic} are also proposed to start with an incomplete library and evolve through generations. 

While these algorithms only discover the equation structure for PDEs with constant coefficients, later works also start to work on the discovery of PDEs with variable coefficients. For PDEs with variable coefficients, we need to determine their PDE structures (the partial derivative terms that form the PDE) and coefficients (the variable coefficients that multiply partial derivative terms in the PDE) at the same time. Sequential Group Threshold Regression \cite{rudy2019data} combines coefficient regression and term selection to find PDEs with variable coefficients. PDE-Net \cite{long2018pde, long2019pde}, Graph-PDE \cite{iakovlev2021learning} and Differential Spectral Normalization \cite{so2021diff} are proposed to use neural blocks such as convolution to discover the PDEs models. In addition, DLrSR \cite{li2020robust} solves the noise problem by separating the clean low-rank data and outliers. A-DLGA \cite{xu2020deep} proposes to alleviate data linear dependency at the sacrifice of estimation error. Up until now, the state-of-the-art approaches have proven to discover some PDEs with variable coefficients, but the discovery of PDEs with highly nonlinear coefficients remains a challenge~\cite{rudy2019data, xu2020deep, long2019pde} due to the overfitting of the sparse regressions and data quality issues. 

\section{Preliminaries} \label{sec:preliminary}

In this section, we introduce the problem studied in this work, and analyze the difficulty of highly nonlinear parametric PDE discovery.

\subsection{Problem Definition}

A physical field dataset $u(x,y,t)$ is defined with respect to some input coordinates $(x, y, t)$, where $x \in [1,...,n]$ and $y \in [1,...,m]$ are spatial coordinates and $t \in [1,...,h]$ is a temporal coordinate.
An example of physical field data is shown in the observation data in Fig. \ref{fig:task}.
We consider the task of discovering two kinds of PDEs: (1) PDEs with constant coefficients, PDEs-CC; and (2) PDEs with variable coefficients, PDEs-VC.
For simplicity, partial derivative terms are denoted by forms like $u_x$ and $u_{xx}$, which are equivalent to $\frac{\partial u}{\partial x}$ and $\frac{\partial^2 u}{\partial x^2}$.
The time derivatives such as $u_t$ (i.e., $\frac{\partial u}{\partial t}$) of a PDE nearly always exist \cite{xu2020dlga}, therefore we follow prior works and set $u_t$ as the regression label.
Let $p$ denote the number of partial derivative candidate terms considered in the task.

\begin{definition}[PDEs with constant coefficients, PDEs-CC]\label{def1}
PDEs-CC are the simplest PDEs, whose coefficients $\xi_i$ are fixed along all coordinates:
\begin{equation} \label{eq:def1}
u_t=\sum_{i=1}^{p}\Theta(u)_i\xi_i, \, \Theta(u)_i \in [1, u,u_x,u_y,u_{xx},..., uu_x, ...]. 
\end{equation}
\end{definition}

\begin{definition}[PDEs with variable coefficients, PDEs-VC]\label{def2}
The coefficients of PDEs-VC are changing in some dimensions, e.g., the spatial dimensions:
\begin{equation} \label{eq:def2}
u_t=\sum_{i=1}^{p}\Theta(u)_i\xi_i(x,y), \Theta(u)_i \in [1, u,u_x,u_y,u_{xx},..., uu_x, ...].
\end{equation}
\end{definition}

A simple example of explicit function is $\xi_{i}(x,y)=\sin x+\cos y$ and other $\xi_{i}(x,y)$ may be anistropic random fields \cite{zhang2001stochastic} that are hard to express by explicit functions.

We can see that a PDE has two parts: the set of $\Theta(u)_i$ for $\forall i$ is the PDE structure, while the set of $\xi_i(x,y)$ for $\forall i$ is the PDE coefficients.
Here, each $\Theta(u)_i$ represents a monomial basis function of $u$ or the combination of two monomial basis functions of $u$.
We consider monomial basis functions only up to the third derivative since higher-order derivatives can be inaccurate due to differential precision \cite{rudy2017data}.
In Eqs. (\ref{eq:def1}-\ref{eq:def2}), the coefficient $\xi(x,y)$ changes w.r.t. spatial coordinates $x$ and $y$. 
In this paper, we discuss the case of spatial variations. If the task is to capture variations in the temporal dimension, we can simply replace $\xi(x,y)$ with $\xi(t)$.

Accordingly, the goal of PDE discovery is to determine: 
\begin{itemize}
\item {\bf Terms}: which coefficient $\xi_i$ is nonzero so that the term $\Theta(u)_i$ exists in the PDE structure;
\item {\bf Coefficients}: the exact values of all nonzero coefficients at each spatial coordinate.
\end{itemize}

Naturally, the accuracy of coefficient estimation would affect the correctness of determining which coefficient is nonzero. This coupling motivates us to choose methods that can perform structure learning and coefficient estimation simultaneously (e.g., sparse regression). Moreover, since the simplicity of PDE is important, we are looking for the PDE with the fewest terms. For example, $u_t=u_x$ is simpler than $u_t=u_x+u_y$ under similar data fitting.

\begin{figure*}
\centering
\includegraphics[width=0.82\linewidth]{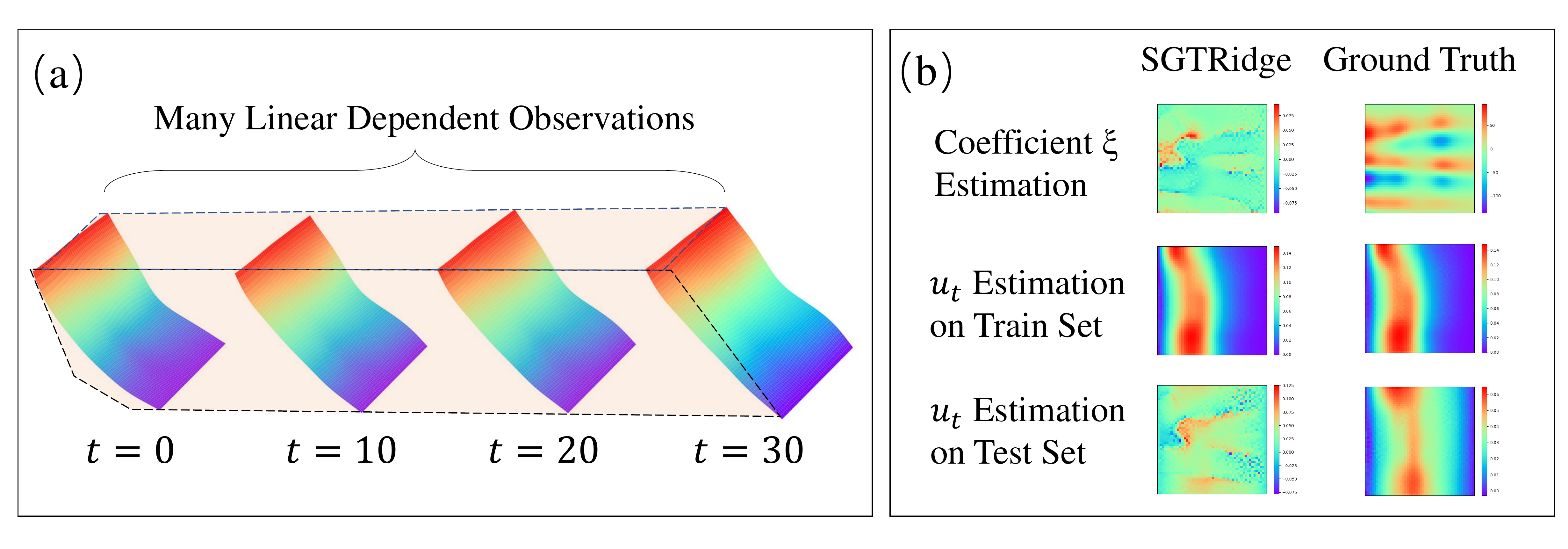}
\caption{The linear dependent observations and data quality issues cause the overfitting of baselines such as SGTRidge \cite{rudy2019data}. The estimated PDE coefficients are fairly irregular, and cannot match the ground truth. Although it can fit the training data well, it fails to generalize to the test data.}
\label{fig:prior}
\end{figure*}

\subsection{Sparse Regression for PDE Discovery}

Sparse regression is widely adopted in previous works to estimate both the terms and coefficients of PDEs \cite{xu2019dl,rudy2017data,li2020robust,long2018pde}.
For parametric PDEs with variable coefficients across spatial dimensions, many linear regressions are separately performed for coefficients at different spatial coordinates $(x,y)$:
\begin{align} \label{eq:saprse1}
Y &= XW + \epsilon, \quad
\epsilon \sim \eta \mathcal{N}(0, \sigma^2) \in \R^{h},
\end{align}
\begin{align} \label{eq:saprse2}
\widehat{W} &= \argmin\limits_{W}
\norm{Y-XW}_2^2
+ \lambda \norm{W}_0,
\end{align}
where $Y=[Y_1, Y_2, ..., Y_h]^\top \in \R^h$ denotes $u_t$ of all the $h$ samples along the temporal dimension, $X_{ji}$ denotes $\Theta(u)_i$ of the $j$-th sample in $X \in \R^{h \times p}$, 
$W=[W_1, W_2, ..., W_p]^\top \in \R^{p}$ denotes all the coefficients $\xi_i$ of the $p$ candidate terms,
and $\epsilon$ denotes the inevitable noise in data.
The above expression describes the scheme where we aim at discovering one PDE from one physical field $u$, which can also extend to the discovering of multiple PDEs from multiple physical fields.
Here, Eqs. (\ref{eq:saprse1}-\ref{eq:saprse2}) repeat $n \times m$ times along the spatial dimensions $x$ and $y$ to get every $\widehat{W}^{[x,y]}$. 

\subsection{The Challenge of Highly Nonlinear Parametric PDE Discovery}

Though the above methods have been effective for some PDEs with simple variable coefficients \cite{rudy2019data, long2018pde, long2019pde, xu2020dlga, xu2020deep}, they still have difficulty in discovering PDEs with highly nonlinear coefficients due to overfitting.
To illustrate this, we use the mean absolute error (MAE) to measure the error of target ($u_t$) fitting across training, development, and test sets.
With the correctness of PDE structure and accurate coefficient estimation, we shall obtain low target fitting MAE on test sets.
As shown in Fig. \ref{fig:prior} (a) and extensively mentioned in the literature \cite{rudy2017data, zhang2001stochastic, li2020robust}, many physics observations are linearly dependent along the temporal dimension since the coefficient fields that determine the observation are not changing along time. 
Linear-dependent observations make the linear equation $Y=XW$ with $rank(X) \leq p$ an underdetermined system that causes overfitting. Furthermore, data sparsity and noise also impair the data quality and exacerbate the problem. Fig. \ref{fig:prior} (b) shows that the estimated coefficients by baseline sparse regression models such as SGTRidge \cite{rudy2019data} are irregular and cannot match the ground truth, and the estimation of the target $u_t$ cannot generalize to test sets.
The overfitting deviates the model from searching for the correct coefficients and terms, despite its good performance on the training set. Data details of Fig. \ref{fig:prior} are shown in Sec. \ref{sec:results} and App. \ref{sec:appendix_data_statistics}.

\begin{figure*}[t]
\centering
\includegraphics[width=0.98\linewidth]{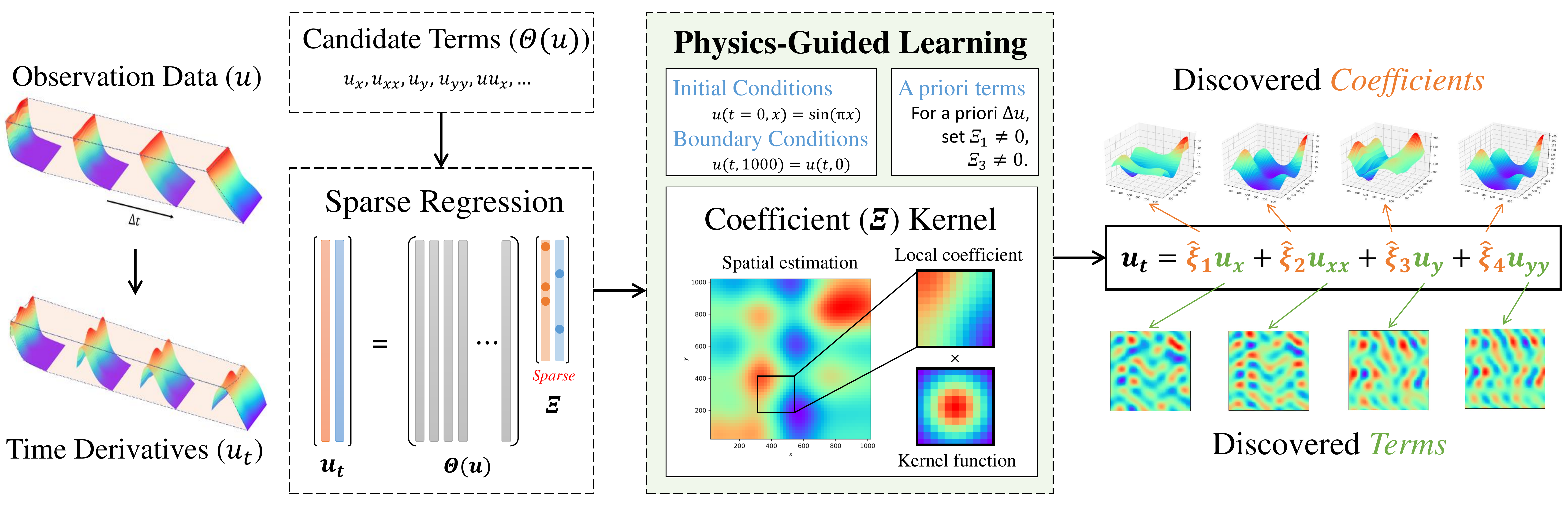}
\caption{Diagram of the HIN-PDE framework.}
\label{fig:task}
\end{figure*}

\section{Methodology} \label{section:method}

In this section, we detail the proposed method, and present some theoretical analysis.
The whole framework is illustrated in Fig. \ref{fig:task}.

\subsection{Physics-Guided Spatial Kernel Estimation} \label{sec:method1}

While various PDE terms and coefficients could overfit the training data, scientists are only interested in the PDE that is interpretable in terms of physics and can stably describe the natural phenomena. In this paper, we incorporate physical principles into the PDE learning model. First, we consider smoothness \cite{zhang2004efficient, zhang2001stochastic}, which is a first principle as PDEs must involve computing derivatives. A "first principle" refers to a basic assumption that cannot be deduced from any other assumption, which is the foundation of theoretical derivation. Here, we state the local smooth principle in Def.~\ref{def:local-smoothing} that ensures the basic accuracy of differentiation. This aligns well with our observation of many physical data, such as the locally smooth coefficient fields in Fig. \ref{fig:coefficient} and the ground-truth coefficients and data in Fig. \ref{fig:prior}. On the contrary, the coefficient estimation and data fitting of SGTRidge are irregular as shown in Fig. \ref{fig:prior}, because the estimation of coefficients is separate at each spatial point, which does not consider the smoothness of coefficients across spatial dimensions. Naturally, we expect that a smooth nonlinear function on spatial dimensions can help model the nonlinear coefficients. 

\begin{definition}[Local smoothness]
\label{def:local-smoothing}
Given coefficient $\xi(x,y)$, the coefficients within a local area (with radius $r$) can be considered as a k-Lipschitz continuous function. Given the spatial distance of any two adjacent coordinates $\text{Dist}=\norm{S(x,y)-S(x',y')} \leq r$ where $S(x,y)$ is the spatial coordinate vector, the slope of the coefficient function is bounded by $\alpha \geq 0$ as $\frac{\left| \xi(x,y)-\xi(x',y') \right|}{\norm{S(x,y)-S(x',y')}} \leq \alpha$. 
\end{definition}

Considering the principles of smoothness, we propose a local kernel estimation in the sparse regression that correlates the coefficient estimation at each spatial coordinate to the adjacent coefficient estimation. A spatially symmetrical kernel (i.e., spatial rotation invariance) for all coordinates (i.e., spatiotemporal translation invariance) would estimate coefficients with respect to conservation laws. A schematic diagram of the physics-guided learning framework is shown in Fig. \ref{fig:task}. We expect it to enhance the model robustness for learning PDEs with highly nonlinear coefficients.

We prove that the proposed sparse regression with local kernel estimation can reduce the coefficient estimation error and reduce the error caused by noise when the coefficient fields comply with the local smoothness principle, with theorems and proofs in the Sec. \ref{sec:proof}. Furthermore, as long as the spatial coordinates of the coefficients are provided, this local kernel estimation is mesh-free for spatiotemporal data, so that nonlinear coefficients can be modeled even with irregularly sparse data.

We denote the spatial coordinate vector as $S(x,y)$ and denote the distance between two spatial coordinates $(x, y)$ and $(x', y')$ as $\norm{S(x, j)-S(x', y')}$. For each $(x,y)$, the proposed model considers all $(x',y')$ that $\norm{S(x, y)-S(x', y')} < r$ to compute
\begin{align} \label{eq:kernel1}
\widehat{W} = \argmin\limits_{W}
\norm{Y-X \Xi}_2^2,
\end{align}
where
\begin{align} \label{eq:kernel2}
\Xi_i^{[x,y]} = 
\frac{\sum K_i^{[x',y']} W_i^{[x',y']}} 
{\sum K_i^{[x',y']}},
\end{align}
\begin{align} \label{eq:kernel3}
K_i^{[x',y']}=\exp(-\frac{D^{[x',y']}}{2\gamma}),
\end{align}
\begin{align} \label{eq:kernel4}
D^{[x',y']}= \norm{S(x,y)-S(x',y')}_2^2,
\end{align}
where $[x,y]$ denotes the spatial coordinate of the estimated coefficient while $[x',y']$ denotes each spatial coordinate of the adjacent coefficients. $r$ denotes the radius of the local area. Note that $W \in \R^{n \times m \times p}$ denotes the model parameters, while $\Xi \in \R^{n \times m \times p}$ is the model weights, i.e., the estimated PDE coefficients $\hat{\xi}$. Here, each $\Xi_i \in \R^{n \times m}$ in $\Xi=[\Xi_1, ..., \Xi_p]$ represents a two-dimensional spatial coefficient of the PDE. $\gamma$ and $q$ are both hyperparameters. The proposed coefficient estimation only takes the spatial distance as input information and is thus mesh-free to apply to any physics data in continuous spaces for use in real practices.

We use the local kernel estimated $\Xi$ of spatially adjacent coefficients instead of $W$ as the regression weight to optimize the model $\widehat{Y}=X \Xi$. The learning of $W_i$ at each $[x',y']$ is dependent $X$ and $Y$ of adjacent coordinates, which allows the model to capture the interaction between adjacent points. Here we can choose Radius Basis Function (RBF) kernel as $K$ that is symmetrical. Therefore, the proposed method leverages the physical principles. 

\subsection{Physics-Guided knowledge Constraints} \label{sec:knowledge}

We further consider the various physics knowledge, which may be useful for a more accurate estimation, such as the initial conditions, boundary conditions and a priori terms of the PDE to be discovered, as shown in Fig. \ref{fig:task}. Essentially, we propose a method to consider their effects as model constraints, which have an impact on the optimization of the loss function
\begin{align} \label{eq:kernel5}
\mathcal{L}(W;u)=\norm{Y-X \Xi}_2^2 + \beta \norm{\hat{u}-u}_2^2 + \lambda \norm{W}_0,
\end{align}
where $Y$ and $X$ are the temporal/spatial derivatives of the observation data $u$ as defined in Sec. 3.1, $\Xi$ is defined in Eqs. (5-8) as the estimated PDE coefficients, a function of $W$. Here, we use $\hat{u}$ to denote the prediction of the observation data based on the discovered PDE (determined by $\Xi$) and some physics knowledge $\Gamma$ (such as initial or boundary conditions). Thus, we can write it as $\hat{u}=f(\Xi, \Gamma)$, where $f$ can be any differentiable PDE solver.

\subsection{Iterative One-Out Sparse Regression}

We use an iterative one-out regression that filters out one $X_{;,i}$ which gives the least Akaike Information Criterion (AIC) \footnote{\url{https://link.springer.com/book/9789027722539}} at each iteration of coefficient estimation. If we use $M$ to denote the set of indexes of reserved coefficients, $M'$ to denote the a priori derivative terms (if there are any), the formula of AIC is used as follows
\begin{equation} \label{eq:regression1}
\text{A}(M) = 2 \sum_{i \in M} 1 - 2 \ln\ \norm{Y-\sum_{i \in M} X_{:,i} \Xi_i}_2^2 .
\end{equation}

The iteration ends when there are only $L$ coefficients left in the regression. This aims to filter out the most irrelevant $\xi_i$ that maximize the least square errors to avoid its intervention in estimating coefficients. The iterative one-out regression repeats
\begin{equation} \label{eq:regression2}
M=M-[i], \  \text{if} \  \text{A}(M-[i]) = \text{min}(\text{A}(M-[j])) \
\text{and} \ i \notin M',
\end{equation}
for $\forall \ j \in M$, and
\begin{align} \label{eq:regression3}
\widehat{W} &= \argmin\limits_{W}
\norm{Y-\sum_{i \in M} X_{:,i} \Xi_i}_2^2 + \beta \norm{\hat{u}-u}_2^2.
\end{align}

Iterative one-out regression is an approximation of sparse group regression (the weight regularization term in Eq. 9). It improves the accuracy in determining the nonzero $\xi_i$ as it avoids the problem with the interference of irrelevant terms in previous works \cite{rudy2019data}. The overall algorithm is expressed in Alg. \ref{alg}.

\begin{algorithm}[t] \label{alg}
    \caption{The physics-guided spatial kernel sparse regression approach.}
    \small
    \label{alg}
    \begin{algorithmic}[1]
        \REQUIRE A priori terms $M'$, target time derivative term $u_t(x,y,t)$ and candidate equation terms $\Theta(u)_i(x,y,t)$ w.r.t. $x \in [1,...,n]$, $y \in [1,...,m]$ and $t \in [1,...,h]$. $p$ that $M=[1,2,...,p]$, $i \in M$, $\lambda$, $\gamma$, $q$, $L$.
        \STATE For convenience, denote $u_t$ as $Y$ and denote $\Theta(u)_i$ as $X_{:,i}$. 
        \WHILE{$size(M) > L$}
            \STATE Compute $\widehat{W}$ by Eq. \ref{eq:regression3}, which is determined by Eqs. (\ref{eq:kernel1}-\ref{eq:kernel5})) in detail, using all $X_{:,i}$ and $Y$ with $i \in M$;
            \STATE Update $M$ based on Eqs. (\ref{eq:regression1}-\ref{eq:regression2});
        \ENDWHILE
        \STATE Compute $\widehat{W}_{best}$ by Eq. \ref{eq:regression3} using all $X_{:,i}$ and $Y$ with $i \in M$;
        \STATE Compute $\widehat{\Xi}_{best}$ by Eq. \ref{eq:kernel2} using $W=\widehat{W}_{best}$.
        \RETURN $M$, $\widehat{\Xi}_{best}$.
    \end{algorithmic}
\end{algorithm}

\subsection{Complexity Analysis} \label{sec:complexity}
In this section, we will show that the time complexity of our spatial kernel estimation scales linearly with regard to size of the dataset.

The size of the dataset can be determined by $n \times m \times h$. There is a hyperparameter $p$, the number of PDE terms. Another hyperparameter is $r$, the radius of the local area. The calculation of the spatial kernel itself has a complexity of $O(1)$. For the estimation of coefficients at each spatial coordinate $(x, y)$, we only allow adjacent data points $(x', y')$ within a local area $\| S(x,y)-S(x',y') \| < r$ to participate in the calculation; therefore, the number of data points involved does not scale up with the size of the dataset, but is determined by the constant $r$. Because we focus on the local smoothness here, $r$ is not a large number. The calculation of each coordinate is $O(hp)$, and there are $n \times m$ numbers of coordinates to calculate for a given dataset. Therefore, the total complexity is $O(nmh)$ if we discard the constant $p$ and the constant related to $r$. In all, the total complexity is linearly proportional to the size of the dataset.

\subsection{Theoretical Analysis} \label{sec:proof}

In this section, we introduce the theoretical analysis to demonstrate the advantages of our model. We provide several theorems in the following with proofs. The proposed spatial kernel estimation (see Eqs. (\ref{eq:kernel1}-\ref{eq:kernel4})) uses the spatial distance to estimate the probability density function of coefficients with nonlinearity. To help understand its advantage, we first introduce spatial averaging estimation with linearity. We intend to show that the spatial averaging estimation can have an estimation error upper bounded by the upper limit of coefficient difference between adjacent coordinates. The coefficient estimation error without such spatial averaging estimation, on the contrary, has no upper bound (i.e., can be fairly large). 

The local averaging estimation is defined as follows. For each $(x,y)$, the spatial averaging estimation considers all $(x',y')$ that $\norm{S(x, y)-S(x', y')} < r$ to compute
\begin{equation}
\widehat{W}_{(avg)} = \argmin\limits_{\Xi}
\norm{Y-X \Xi}_2^2,
\end{equation}
\begin{equation}
\Xi_{(avg)}^{[x,y]} = 
\frac{\sum_{(x',y')} W^{[x',y']}}{\sum_{(x',y')} 1}.
\end{equation}
where $W$ denotes model parameters of the averaging estimation. We use $\Xi_{(avg)}$ to denote the estimated coefficients and $\xi$ to denote the ground-truth coefficients. To understand the error introduced by averaging and/or kernel, we ignore the overfitting issue for the moment and suppose $\hat{W}$ can model $\xi$ perfectly. We introduce the Lipschitz continuity to express the local smoothness with a Lipschitz constant $\alpha \geq 0$, as introduced in Def. \ref{def:local-smoothing}. Here, we consider the upper limit of coefficient difference between adjacent coefficients within the local area for all x, y, x', and y' as 
\begin{equation}
\alpha \geq \frac{|\widehat{W}^{[x',y']} - \widehat{W}^{[x,y]}|}{\norm{S(x,y)-S(x',y')}} \geq \frac{|\widehat{W}^{[x',y']} - \widehat{W}^{[x,y]}|}{r}.
\end{equation}

In the worst case, we have all the coefficients on only one side with differences approximating $\alpha r$. Therefore, the upper bound of estimation error is 
\begin{equation}
\label{avg_error}
sup(|\widehat{\Xi}_{(avg)}^{[x,y]}-\xi^{[x,y]}|) = \widehat{W}^{[x,y]} - \frac{\sum_{(x',y')} \widehat{W}^{[x,y]} - \alpha r}{\sum_{(x',y')} 1} = \alpha r.
\end{equation}

While the spatial averaging coefficient estimation has a upper bound of coefficient estimation error, the spatially independent estimation in Eqs. (\ref{eq:saprse1}-\ref{eq:saprse2}) practiced by many baselines cannot guarantee to match the ground-truth coefficients even if Eq. \ref{eq:saprse2} is optimized due to the existence of many linearly dependent observations. We assume that the spatial averaging estimation can avoid this issue by using extra data from adjacent coordinates within the local area in the sacrifice of introducing the estimation error as described in Eq. \ref{avg_error}. We can also easily demonstrate that the local averaging estimation has a lower estimation error than the strategy practiced in A-DLGA \cite{xu2020deep} that makes coefficients grids coarser by merging grids within each spatial area into one grid, which also uses extra adjacent data to alleviate the issue caused by linearly independent observations. We formalize this in Theorem \ref{theorem1}. 

\begin{theorem}[Reduction on coefficient error by local averaging estimation]
\label{theorem1}
With respect to the local smoothness principle, the coefficients estimated by the spatial averaging estimation has strictly lower upper-bound coefficient estimation error than A-DLGA.
\end{theorem}

\begin{proof}
Assume that the \textit{Local Smoothness Principle} in Definition \ref{def:local-smoothing} applies. For $\xi(x,y)$, we denote the estimated coefficients of A-DLGA as $\widehat{W}^{\dagger}(x,y)$. For each $(x,y)$, A-DLGA considers all $(x',y')$ that $\norm{S(x'', y'')-S(x', y')} < r$, where $\norm{S(x'', y'')-S(x, y)} < r$. The upper bound of estimation error should be the case where $\norm{S(x'', y'')-S(x, y)} = 2r$, so that the upper bound of coefficient difference would be $2 \alpha r$. Therefore, 
\begin{align}
\label{adlga_error}
sup(|\widehat{W}^{\dagger}-\xi|) = \widehat{W}^{[x,y]} - \frac{\sum_{(x',y')} \widehat{W}^{[x,y]} - 2\alpha r}{\sum_{(x',y')} 1} \\
= 2\alpha r > \alpha r = sup(|\widehat{\Xi}_{(avg)}-\xi|).
\end{align}
\end{proof}

Furthermore, we show that our spatial kernel estimation has a lower estimation error than the spatial averaging estimation by Theorem \ref{theorem2}. Our model is more accurate than local averaging estimation, so it is more accurate than A-DLGA \cite{xu2020deep}. 

\begin{theorem}[Reduction on coefficient error by local kernel]
\label{theorem2}
With respect to the local smooth principle, the coefficients estimated by the spatial kernel estimation has strictly lower coefficient estimation error than the spatial averaging estimation.
\end{theorem}
\begin{proof}
Assume that the \textit{Local Smoothness Principle} in Definition \ref{def:local-smoothing} applies, we consider the estimation of coefficient $\xi(x,y)$. Because the coefficient function is a k-Lipschitz continuous function, the coefficient difference increases with spatial distance. We denote the estimated coefficients as $\Xi$. Note that the kernel in Eq. \ref{eq:kernel3} defined as $K^{[x',y']}=\exp(-\frac{\norm{S(x,y)-S(x',y')}_2^2}{2\gamma})$ decreases in the spatial distance, so closer coefficients give more contributions. Thus, 
\begin{align}
|\widehat{\Xi}^{[x,y]}-\xi^{[x,y]}| 
&= 
\frac{\sum K^{[x',y']} (\widehat{W}^{[x',y']} - \widehat{W}^{[x,y]})} 
{\sum K^{[x',y']}} \\
&\propto \frac{\sum K^{[x',y']} \delta(\norm{S(x,y)-S(x',y')})} 
{\sum K^{[x',y']}} \\
&\leq \frac{\sum \delta(\norm{S(x,y)-S(x',y')})} 
{\sum 1} \\
&= |\widehat{\Xi}_{(avg)}^{[x,y]}-\xi^{[x,y]}|.
\end{align}
The above equations and inequalities prove that the coefficient estimation error of spatial kernel estimation is strictly lower than the coefficient estimation error of spatial averaging estimation.
\end{proof}

\begin{table*}[t]
\centering
  \caption{Model performance of HIN-PDE under different noisy levels on PDEs with constant coefficients, evaluated by the recall rate, coefficient errors and fitting errors. All equations are correctly discovered, even with noise in data. }
  \label{tab:cc}
  \begin{tabular}{c c c c c c c c c c}
    \toprule
    Metrics & \multicolumn{3}{c}{Recall (\%)} & \multicolumn{3}{c}{Coef. Error ($\times 10^{-3})$} & \multicolumn{3}{c}{Fitting Error ($\times 10^{-3})$} \\ 
    \midrule
    Noise Level & 0\% & 10\% & 20\% & 0\% & 10\% & 20\% & 0\% & 10\% & 20\%\\ 
    \midrule
    Burgers' Equation & 100 & 100 & 100 & 2.603 & 6.124 & 6.946 & 0.205 & 0.356 & 1.004 \\
    \midrule
    KdV Equation & 100 & 100 & 100 & 1.417 & 7.385 & 14.36 & 3.729 & 187.8 & 375.5\\
    \midrule
    C-I Equation & 100 & 100 & 100 & 3.623 & 12.69 & 25.38 & 1.691 & 11.85 & 23.71\\
    \bottomrule
  \end{tabular}
\end{table*}

Moreover, the spatial kernel estimation reduces the coefficient estimation error caused by noise, which is proved in Theorem \ref{theorem3}. For a coefficient at a spatial coordinate $(x, y)$, its estimation is affected by both the noise in $Y$ and the noises in $Y$ at adjacent coordinates within the local area. We only discuss the estimation error caused by noise here and assume $\widehat{\Xi}=\widehat{W}=\xi$ if $\eta=0, \epsilon=0$. This allows us to discuss the noise effect alone without the error caused by the kernel discussed in Theorem \ref{theorem2}. We prove that the weighted addition of independent Gaussian noises by kernel has a lower error, and the error decreases with the number of adjacent coefficients.

\begin{theorem}[Reduction on coefficient error caused by noise]
\label{theorem3}
Assume that the coefficient estimation error is only caused by noise so that $\widehat{\Xi}=\widehat{W}=\xi$ if $\eta=0, \epsilon=0$, then we must have $|\widehat{\Xi} - \xi| < |\widehat{W} - \xi|$ if $\eta \neq 0, \epsilon \sim \eta \mathcal{N}(0, \sigma^2) \in \R^{h}$. 
\end{theorem}
\begin{proof}
Consider $\epsilon \sim \eta \mathcal{N}_h(0, \sigma^2)$ in estimation that $\widehat{W} = (X^TX+\lambda I)X^T(Y+\epsilon)$. $|\widehat{W} - \xi| = (X^TX+\lambda I)X^T\epsilon$. We assume $\widehat{\Xi}=\widehat{W}=\xi$ if $\eta=0, \epsilon=0$, which means for each $\widehat{\Xi}^{[x',y']}$ there must always be another $\widehat{\Xi}^{[x'',y'']}$ such that 1) they have the same distance to $\widehat{\Xi}^{[x,y]}$ so they have the same kernel value, i.e. $\norm{S(x',y')-S(x,y)}=\norm{S(x'',y'')-S(x,y)}$, and 2) they are symmetrical to the value of $\xi^{[x,y]}$ so that their biases can be offset, i.e. $\xi^{[x, y]}-\xi_i^{[x', y']} = \xi^{[x'', y'']}-\xi^{[x, y]}$. Then, for all $(x',y')$ that $\norm{S(x, y)-S(x', y')} < r$, the estimated coefficients should be
\begin{align}
|\widehat{\Xi}^{[x,y]} - \xi^{[x, y]}| &=
\frac{\sum_{(x',y')} K^{[x',y']} \epsilon_{(x',y')}}
{\sum_{(x',y')} K^{[x',y']}}.
\end{align}
For each $\epsilon_{(x',y')} \sim \eta N(0, \sigma^2)$ that is i.i.d., 
as $\sum \epsilon \sim \eta N(0, \sum \sigma_i^2)$, we have $|\widehat{\Xi}^{[x,y]} - \xi^{[x, y]}| \sim \eta N(0, \sum_{(x',y')} \frac{K^{[x',y']}\sigma^2}{K^{[x',y']}})$.
As each $\sigma$ is the same for all adjacent coefficients within the area, we have $|\widehat{\Xi}^{[x,y]} - \xi^{[x, y]}| \sim \frac{\eta}{\sum_{(x',y')} 1} N(0, \sigma^2)$. However, $|\widehat{W}^{[x,y]} - \xi^{[x, y]}| \sim \eta N(0, \sigma^2)$.
Therefore, $|\widehat{\Xi} - \xi| < |\widehat{W} - \xi|$ and $\frac{|\widehat{W} - \xi|}{|\widehat{\Xi} - \xi|} = \sum_{(x',y')} 1$. 
\end{proof}

\begin{table*}[t]
\centering
  \caption{Model comparison among SGTRidge, PDE-Net, A-DLGA and HIN-PDE on PDEs with variable coefficients. The percentage numbers indicate the noisy levels in the data. Among the compared methods, HIN-PDE is the only one that is able to correctly discover highly nonlinear parametric PDEs.}
  \label{tab:fitting}
  \begin{tabular}{c c c c c c c c}
    \toprule
    \multirow{2}*{Dataset} & \multirow{2}*{Metric} & SGTRidge & PDE-Net & A-DLGA & \multicolumn{3}{c}{HIN-PDE} \\ 
    ~ & ~ & 5\% & 5\% & 5\% &  5\% & 10\% & 20\%\\
    \midrule
    \multirow{4}*{1-HNC} & Train Err $(\times 10^{-3})$ & $1.148$ & $2.190$ & $16.70$ & $3.314$ & $6.283$ & $11.88$ \\
    ~ & Dev Err  $(\times 10^{-3})$ & $3.907$ & $10.85$ & $47.39$ & $3.919$ & $6.373$ & $12.10$ \\ 
    ~ & Test Err  $(\times 10^{-3})$ & $28.25$ & $31.80$ & $136.8$ & \bf{3.686} & $6.367$ & $12.29$\\ 
    ~ & Recall (\%) & 50 & 50 & 25 & \bf{100} & 100 & 75  \\ 
    \midrule
    \multirow{4}*{2-HNC} & Train Err  $(\times 10^{-3})$  & $2.283$ & $2.697$ & $19.83$ & $3.794$ & $5.676$ & 8.329\\
    ~ & Dev Err  $(\times 10^{-3})$ & $26.87$ & $42.23$ & $43.04$ & $3.794$ & $5.674$ & 8.332\\ 
    ~ & Test Err $(\times 10^{-3})$ & $106.1$ & $169.2$ & $123.8$ & \bf{3.585} & $5.499$ & 8.318 \\ 
    ~ & Recall (\%) & 25 & 25 & 25 & \bf{100} & 100 & 75 \\ 
    \midrule
    \multirow{4}*{3-HNC} & Train Err  $(\times 10^{-3})$ & $0.134  $ & $0.129  $ & $1.033  $ & $0.331  $ & $0.583$  & $0.969$\\
    ~ & Dev Err  $(\times 10^{-3})$ & $1.588  $ & $1.563  $ & $2.661  $ & $0.342 $ & $0.588$ & $0.997$  \\ 
    ~ & Test Err  $(\times 10^{-3})$ & $9.095  $ & $9.005  $ & $7.872  $ & \bf{0.343} & $0.589$ & $1.018$ \\ 
    ~ & Recall (\%) & 25 & 25 & 50 & \bf{100} & 100 & 75  \\ 
    \midrule
    \multirow{4}*{4-HNC} & Train Err  $(\times 10^{-3})$ & $0.336  $ & $0.301  $ & $21.50  $ & $1.733  $  & 2.235 & 3.662\\
    ~ & Dev Err  $(\times 10^{-3})$ & $11.52  $ & $10.13  $ & $37.61  $ & $1.729  $ & 2.302 & 4.037 \\ 
    ~ & Test Err  $(\times 10^{-3})$ & $94.47  $ & $85.60  $ & $150.0  $ & \bf{1.703} & 2.521 & 7.505\\ 
    ~ & Recall (\%) & 0 & 0 & 0 & \bf{100} & 75 & 25  \\ 
    \midrule
    \multirow{4}*{5-HNC} & Train Err  $(\times 10^{-3})$ & $1.218  $ & $1.506  $ & $20.66  $ & $1.940  $ & 3.139 & 4.438\\
    ~ & Dev Err  $(\times 10^{-3})$ & $6.624  $ & $7.508  $ & $42.72  $ & $1.984  $ & 2.984 & 4.254\\ 
    ~ & Test Err  $(\times 10^{-3})$ & $13.63  $ & $15.17  $ & $109.6  $ & \bf{1.733} & 3.049 & 4.345  \\ 
    ~ & Recall (\%) & 50 & 50 & 50 & \bf{100} & 100 & 100  \\ 
    \bottomrule
  \end{tabular}
\end{table*}

\section{Experiments} \label{sec:experiments}
We conduct experiments\footnote{\url{https://github.com/yingtaoluo/Highly-Nonlinear-PDE-Discovery}} on PDEs with both constant and variable coefficients to demonstrate the effectiveness of HIN-PDE.

\subsection{Experimental Setting} \label{subsec:setting}

We first introduce the settings of our experiemtns.

\subsubsection{Setup}
Our experiments aim to discover PDEs terms and coefficients.
The proposed method is compared with PDE-net \cite{long2018pde}, Sparse Regression (we compare with SGTRidge \cite{rudy2019data} here and SINDy \cite{brunton2016discovering, champion2019data} is also an example of sparse regression) and A-DLGA \cite{xu2020deep}. We split the first 30\% data in the time axis as the training set, the next 30\% data as the development set, and the last 40\% as the test set.
In App. \ref{sec:appendix_hyperparameter_robustness}, for each model on each dataset, we tune the hyperparameters, i.e. $\gamma$, $q$ and $\lambda$, via grid search, so that it has the lowest target $u_t$ fitting error on the development set. 

\subsubsection{Datasets}
To verify how well the proposed model performs on the discovery of PDEs-CC, we consider the Burgers' equation, the Korteweg-de Vries (KdV) equation, and the Chaffe-Infante equation. 
For PDEs-VC, we consider the governing equation of underground seepage, in which we have five cases are spatiotemporal 3D PDEs with different highly nonlinear coefficient fields (see Fig. \ref{fig:coefficient}(d)) that are hard to express in mathematics explicitly. 
The data details of PDEs-VC are provided in App. \ref{sec:appendix_data_statistics}.

\subsubsection{Evaluation Metrics}
We adopt three metrics for evaluation: 
\begin{enumerate}
\item Recall of the discovered PDE terms compared to ground truth;
\item The mean absolute error of coefficient $\xi$ estimation;
\item The mean absolute error of target $u_t$ fitting.
\end{enumerate}
The recall rate of PDE terms and the coefficient estimation error indicate whether the discovered PDE is close to the ground-truth PDE that can generalize to future responses with the correct physical mechanism. Target fitting error tests how well the discovered PDE generalizes to the target $u_t$.
Moreover, we inject levels of noise to the data to verify the robustness of PDE discovery methods.

\subsection{Results on PDEs with constant coefficients} \label{sec:results}

We first investigate whether HIN-PDE can correctly discover PDEs with constant coefficients, under different levels of noise.
Results are shown in Tab. \ref{tab:cc}.
Specifically, we use irregular samples to simulate sparsity in the continuous space.
We observe from the results that, our method performs well even with noisy, 3D, irregularly sampled data and multiple physical fields. 

\subsubsection{Burgers’ Equation}
We consider the discovery of the spatiotemporal 3D Burgers' equation with two physical fields $u$ and $v$. For the sparse regression, we prepare a group of candidate functions that consist of polynomial terms $\{1, u, v, u^2, uv, v^2\}$, derivatives $\{1, u_x, u_y, v_x, v_y, \Delta u, \Delta v\}$ and their combinations.
The ground-truth formulation of the Burgers’ equation is
\begin{equation}
\begin{aligned}
u_t &= 0.005 u_{xx} + 0.005 u_{yy} - uu_x - vu_y,\\
v_t &= 0.005 v_{xx} + 0.005 v_{yy} - uv_x - vv_y.
\end{aligned}
\end{equation}
Following the physics-guided learning, we set diffusion terms as known a priori. The dimensionality of the dataset is $100 \times 100 \times 200$. We irregularly sample 10000 data and add 10\% Gaussian noise. As shown in Tab. \ref{tab:cc}, the Burgers’ equation is accurately discovered, under different noisy levels.

\subsubsection{Korteweg-de Vries (KdV) Equation}
We consider the discovery of spatiotemporal 2D Korteweg-de Vries (KdV) equation. We prepare a group of candidate functions that consist of polynomial terms $\{1, u, u^2\}$, derivatives $\{1, u_x, u_{xx}, u_{xxx}\}$ and their combinations.
The ground-truth formulation of the KdV equation is
\begin{equation}
u_t = - uu_x - 0.0025 u_{xxx}.
\end{equation}
The dimensionality of the dataset is $512 \times 201$. We irregularly sample 5000 data and add 10\% Gaussian noise. As shown in Tab. \ref{tab:cc}, the KdV equation is accurately discovered, under different levels of noise.

\subsubsection{Chaffe-Infante (C-I) Equation}
We consider the discovery of spatiotemporal 2D Chaffe-Infante equation. We prepare a group of candidate functions that consist of polynomial terms $\{1, u, u^2, u^3\}$, derivatives $\{1, u_x, u_{xx}\}$ and their combinations.
The ground-truth formulation of the C-I equation is
\begin{equation}
u_t = u_{xx} - u + u^3.
\end{equation}
The dimensionality of the dataset is $301 \times 201$. We irregularly sample 5000 data and add 10\% Gaussian noise. As shown in Tab. \ref{tab:cc}, the C-I equation is accurately discovered, under different levels of noise.

\begin{figure}[t]
\centering
\includegraphics[width=1\linewidth]{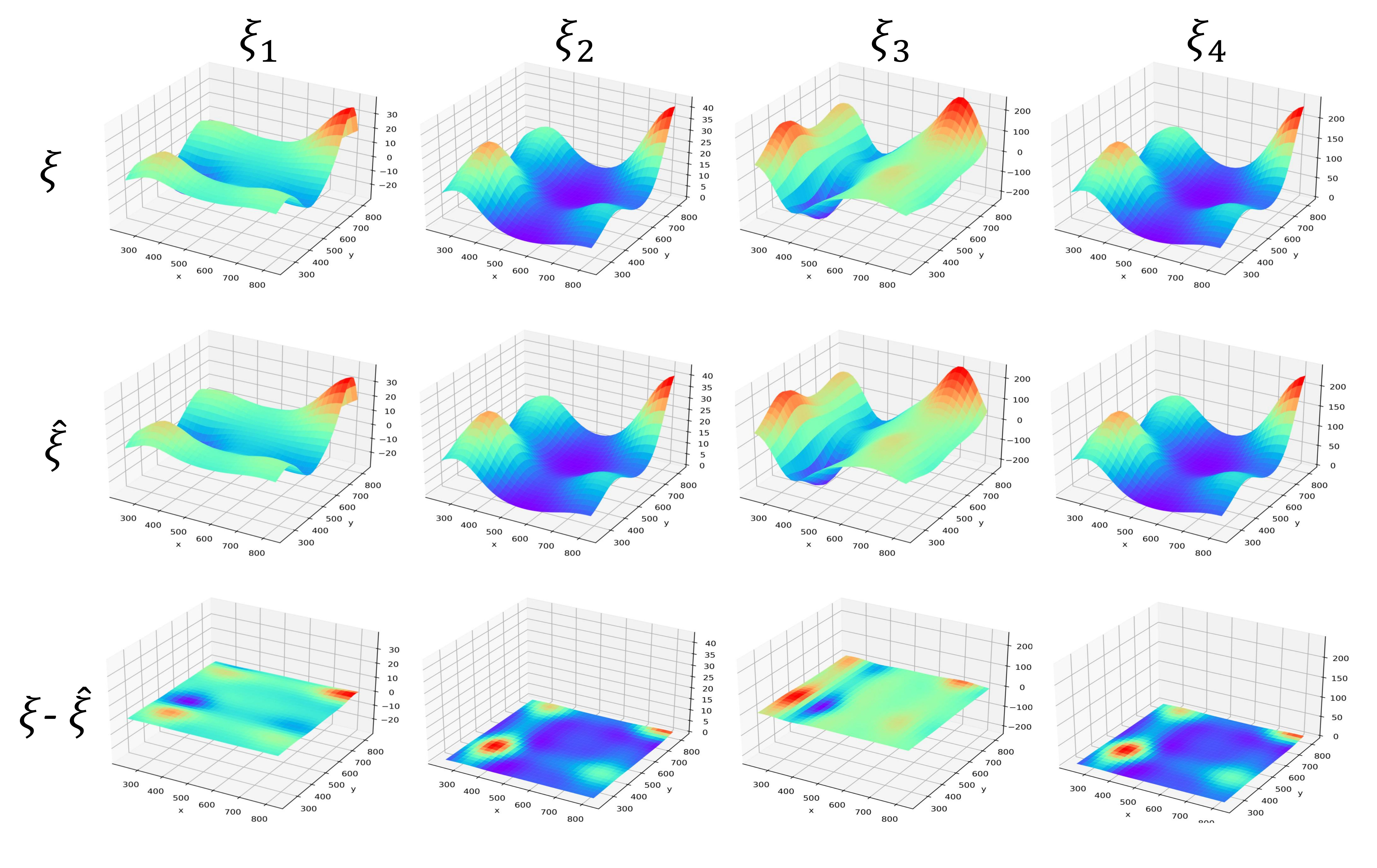}
\caption{Comparison of estimated and correct nonzero coefficients of 1-HNC. Three rows represent the results of Ground Truth, HIN-PDE and their residual errors, respectively. Four columns represent four $\xi_i$ that are nonzero in the ground-truth PDE queation. We have the residual errors as $|\xi-\hat{\xi}|/|\xi|=1\%$, which indicates the errors of estimated coefficients are extremely small, and HIN-PDE accurately fits the PDEs with highly nonlinear coefficients.}
\label{est_coef}
\end{figure}

\subsection{Results on PDEs with variable coefficients} \label{sec:results}

Then, we conduct performance comparison on PDEs with variable coefficients.
We consider the discovery of the spatiotemporal 3D governing equation of underground seepage with five different highly nonlinear variable coefficients, namely 1-HNC, 2-HNC, 3-HNC, 4-HNC and 5-HNC.
We prepare a group of candidate functions that consist of polynomial terms $\{1, u, u^2\}$, derivatives $\{1, u_x, u_{xx}, u_{xxx}\}$ and their combinations.
The correct formulation shall be
\begin{equation}
u_t = \hat{\xi}_1 u_x + \hat{\xi}_2 u_y + \hat{\xi}_3 u_{xx} + \hat{\xi}_4 u_{yy},
\end{equation}
where $\xi$ is the coefficient that can not be explicitly expressed.
And more details can be found in App. \ref{sec:appendix_data_statistics}.
The dimensionality of the fives cases are all $50 \times 50 \times 51$. We irregularly sample 10000 data and add 5\% Gaussian noise.
The recall of terms and target $u_t$ fitting errors are shown in Tab. \ref{tab:fitting}. Our model can discover the terms correctly from noisy and irregularly sampled sparse data and can generalize to future data for all the five cases with highly nonlinear coefficients.
The coefficient estimation of 1-HNC is visualized in Fig. \ref{est_coef} as an example, with the visualizations of more cases in App. \ref{sec:more} showing that the relative coefficient estimation error is no more than 1\%.
Apparently, our method is the only one that can correctly discover PDEs with variable coefficients.
These results strong demonstrate the effectiveness of HIN-PDE, and for the first time we are able to discover highly nonlinear parametric PDEs.

The discovered PDEs contain convection terms and diffusion terms along spatial dimensions, which align well with the ground-truth PDEs of underground seepage derived from the conservation of mass and Darcy's law \cite{cooper1966equation}. On the contrary, all the other baselines render false PDE terms; the test target fitting errors of baselines are much larger than their training errors, reflecting overfitting. The test fitting errors of our model are much smaller than baselines, showing that our model effectively reduces the estimation error. To investigate the robustness of our method, we include results under noisy levels from 5\% to 20\%. In most previous works for PDEs-CC \cite{rudy2017data, champion2019data, rao2021discovering} or PDEs-VC \cite{rudy2019data, long2018pde, li2020robust}, the model robustness to 5\% or 10\% noise levels is verified. Comparatively, the noise scale we test for our model is fairly large. We show that our model performs well for most cases under 10\% noise. When the noisy level goes up to 20\%, in some cases one out of four PDE terms discovered would be wrong as
\begin{equation}
u_t = \hat{\xi}_1 \underline{uu_x} + \hat{\xi}_2 u_y + \hat{\xi}_3 u_{xx} + \hat{\xi}_4 u_{yy},
\end{equation}
or
\begin{equation}
u_t = \hat{\xi}_1 u_x + \hat{\xi}_2 u_y + \hat{\xi}_3 u_{xx} + \hat{\xi}_4 \underline{uu_{yy}},
\end{equation}
where the underlined terms are the wrongly predicted terms.
Moreover, we find that even under extremely large noise, our model can discover PDEs that can generalize well to future data on test sets, since $uu_x$ and $u_x$ are very similar when $u$ is not rapidly changing along spatial dimensions.
Moreover, we find that the model performs well in a wide range of hyperparameters, with details in App. \ref{sec:appendix_hyperparameter_robustness}. Overall, our model shows excellent robustness against overfitting, especially with sparse and noisy data.

\section{Conclusion and Future Work} \label{sec:conclusion}

How to discover Partial Differential Equations (PDEs) with highly nonlinear variable coefficients from sparse and noisy data is an important task.
To address the overfitting of coefficients caused by data quality issues in previous baselines, we propose a physics-guided spatial kernel estimation in sparse regression that aligns well with the local smooth principle in PDEs and conservation laws.
The proposed model incorporates physical principles into a nonlinear smooth kernel to model the highly nonlinear coefficients.
We theoretically prove that it strictly reduces the coefficient estimation error of previous baselines and is also more robust against noise.
With spatial coordinates of coefficients, the model can apply to mesh-free spatiotemporal data without grids.
In experiments, it demonstrates the ability to find various PDEs from sparse and noisy data.
More importantly, it for the first time reports the discovery of PDEs with highly nonlinear coefficients, while previous baselines yield false results.
Our model performs well with a wide range of hyperparameters and noise level up to 20\%. With the state-of-the-art performance, our method brings hope to discover complex PDEs that comply with the continuously differentiable and local smoothness principles to help scientists understand unknown complex phenomena.

In the future, how to avoid the intervention of correlated similar terms and improve the accuracy of differentiation remain important. Our method works for PDEs that comply with the principles, but may remain intractable for more rarely complex coefficient fields. Also, how to discover equations without the prior knowledge of time-dependent target term is not discussed yet.

\begin{acks}
This work is partially supported by National Natural Science Foundation of China (62206291 and 62106116).
\end{acks}

\clearpage
\bibliographystyle{ACM-Reference-Format}
\balance
\bibliography{paper}

\clearpage

\appendix

\section{Data statistics}
\label{sec:appendix_data_statistics}
We introduce the governing equation of underground seepage in the following. The subsurface flows with different coefficients are taken as (1-5)-HNCs. The governing equation for the data is:
\begin{align}
S_s \frac{\partial u}{\partial t} = \frac{\partial}{\partial x}(K(x,y)\frac{\partial u}{\partial x}) + \frac{\partial}{\partial y}(K(x,y)\frac{\partial u}{\partial y}).
\end{align}
where $S_s$ denotes the specific storage; $K(x, y)$ denotes the hydraulic conductivity field; and $u$ denotes the hydraulic head. $u$ is the physical field and $K$ is the coefficient field. The same equation is also used by PDE-Net \cite{long2018pde} but its coefficient field is much simpler. The hydraulic conductivity field $K(x, y)$ in the governing equation is set to be heterogeneous to simulate real situations in practice, which is random fields with higher complexity following a specific distribution with corresponding covariance \cite{zhang2004efficient, huang2001convergence, zhang2001stochastic, wang2020deep}. 

In detail, a two-dimensional transient saturated flow in porous medium is considered. The domain is evenly divided into $51 \times 51$ grid blocks and the length in both directions is 1020 [L], where [L] denotes a length unit. The left and right boundaries are set as constant pressure boundaries and the hydraulic head takes values of $H_{x=0}=202$ [L] and $H_{x=1020}=200$ [L], respectively. Furthermore, the two lateral boundaries are assigned as no-flow boundaries. The specific storage is set as 0.0001. The total simulation time is 10 [T], where [T] denotes any consistent time unit, with each time step being 0.2 [T], resulting in 50 time steps. The initial conditions are $H_{t=0, x=0}=202$ [L] and $H_{t=0, x \neq 0}=200$ [L]. The mean and variance of the log hydraulic conductivity are given as 0 and 1, respectively. 
In addition, the correlation length of the field is $408$ [L]. The hydraulic conductivity field is parameterized through KLE with 20 basis terms. An example of conductivity field is shown in Fig. \ref{S1}(a), which exhibits strong anisotropy. The MODFLOW software is adopted to perform the simulations to obtain the dataset as exemplified in Fig. \ref{S1}(b) and (c).

\begin{figure}[H]
\centering
\includegraphics[width=1\linewidth]{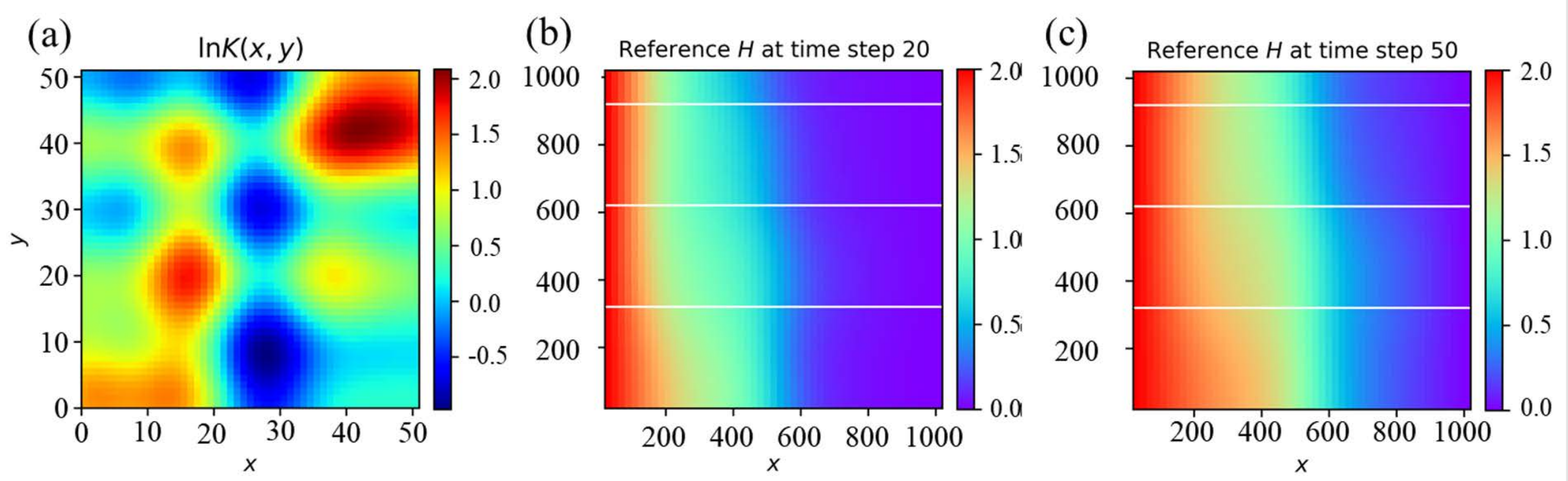}
\caption{Hydraulic conductivity (K) pressure field.}
\label{S1}
\end{figure}


\section{Hyperparameter Study}
\label{sec:appendix_hyperparameter_robustness}
In addition to the main experiment and the robustness experiment, we also conduct a hyperparameter analysis. We set the radius $r$ within [2, 5, 10] and set the $\gamma$ value of the Gaussian kernel within [0.03, 0.1, 0.3, 1].
We show the full results of hyperparameter analysis on (1-5)-HNCs in Tabs. \ref{hyper1}-\ref{hyper5} respectively.
When $\gamma \rightarrow 0$, the kernel estimation is equivalent to the local averaging introduced in Sec. \ref{sec:proof}. When $\gamma \rightarrow \infty$, the kernel estimation degrades to separate regression at each coordinate. The value of $\gamma$ decreases when the coefficient error increases, which aligns well with Theorems \ref{theorem2} and \ref{theorem3} in Sec. \ref{sec:proof} that local averaging estimation has larger error. 

A wide range of hyperparameters can all give the correct PDE structure. The optimal values of both $r$ and $\gamma$ should be tuned in real practice. If the radius is too large, the kernel will not be “local” to match the principle. The hyperparameters we use in the main experiments are $r=10, \gamma=1$.

\begin{table}[H]
\centering
  \caption{Performance w.r.t. hyperparameters on 1-HNC.}
  \label{hyper1}
  \begin{tabular}{c c c}
    \toprule
    Hyperparameters & Recall & Coefficient error \\
    \midrule
    $r=10, \gamma=1$ & 100\% & 0.7090 \\
    $r=10, \gamma=0.3$ & 100\% & 2.2296 \\
    $r=10, \gamma=0.1$ & 100\% & 5.2279 \\
    $r=10, \gamma=0.03$ & 100\% & 8.2493 \\
    $r=5, \gamma=1$ & 100\% & 0.5962 \\
    $r=5, \gamma=0.3$ & 100\% & 1.4790 \\
    $r=5, \gamma=0.1$ & 100\% & 2.0122 \\
    $r=5, \gamma=0.03$ & 100\% & 2.2271 \\
    $r=2, \gamma=1$ & 100\% & 1.7029 \\
    $r=2, \gamma=0.3$ & 100\% & 2.6926 \\
    $r=2, \gamma=0.1$ & 100\% & 3.0050 \\
    $r=2, \gamma=0.03$ & 100\% & 3.1153 \\
    \bottomrule
  \end{tabular}
\end{table}

\begin{table}[H]
\centering
  \caption{Performance w.r.t. hyperparameters on 2-HNC.}
  \label{hyper2}
  \begin{tabular}{c c c}
    \toprule
    Hyperparameters & Recall & Coefficient error \\
    \midrule
    $r=10, \gamma=1$ & 100\% & 0.1183 \\
    $r=10, \gamma=0.3$ & 100\% & 0.3749 \\
    $r=10, \gamma=0.1$ & 100\% & 0.8757 \\
    $r=10, \gamma=0.03$ & 100\% & 1.3561 \\
    $r=5, \gamma=1$ & 100\% & 0.5962 \\
    $r=5, \gamma=0.3$ & 100\% & 1.4790 \\
    $r=5, \gamma=0.1$ & 100\% & 2.0122 \\
    $r=5, \gamma=0.03$ & 100\% & 2.2271 \\
    $r=2, \gamma=1$ & 100\% & 1.7029 \\
    $r=2, \gamma=0.3$ & 100\% & 2.6926 \\
    $r=2, \gamma=0.1$ & 100\% & 3.0050 \\
    $r=2, \gamma=0.03$ & 100\% & 3.1153 \\
    \bottomrule
  \end{tabular}
\end{table}

\begin{table}[H]
\centering
  \caption{Performance w.r.t. hyperparameters on 3-HNC.}
  \label{hyper3}
  \begin{tabular}{c c c}
    \toprule
    Hyperparameters & Recall & Coefficient error \\
    \midrule
    $r=10, \gamma=1$ & 100\% & 0.0272 \\
    $r=10, \gamma=0.3$ & 100\% & 0.0874 \\
    $r=10, \gamma=0.1$ & 100\% & 0.2093 \\
    $r=10, \gamma=0.03$ & 100\% & 1.3561 \\
    $r=5, \gamma=1$ & 100\% & 0.0206 \\
    $r=5, \gamma=0.3$ & 100\% & 0.0514 \\
    $r=5, \gamma=0.1$ & 100\% & 0.0701 \\
    $r=5, \gamma=0.03$ & 100\% & 0.0776 \\
    $r=2, \gamma=1$ & 100\% & 0.0933 \\
    $r=2, \gamma=0.3$ & 100\% & 0.1478 \\
    $r=2, \gamma=0.1$ & 100\% & 0.1650 \\
    $r=2, \gamma=0.03$ & 100\% & 0.1711 \\
    \bottomrule
  \end{tabular}
\end{table}

\begin{table}[H]
\centering
  \caption{Performance w.r.t. hyperparameters on 4-HNC.}
  \label{hyper4}
  \begin{tabular}{c c c}
    \toprule
    Hyperparameters & Recall & Coefficient error \\
    \midrule
    $r=10, \gamma=1$ & 100\% & 0.0416 \\
    $r=10, \gamma=0.3$ & 100\% & 0.1317 \\
    $r=10, \gamma=0.1$  & 100\% & 0.3103 \\
    $r=10, \gamma=0.03$  & 100\% & 0.4879 \\
    $r=5, \gamma=1$  & 100\% & 0.0330 \\
    $r=5, \gamma=0.3$  & 100\% & 0.0823 \\
    $r=5, \gamma=0.1$  & 100\% & 0.1121 \\
    $r=5, \gamma=0.03$  & 100\% & 0.1242 \\
    $r=2, \gamma=1$  & 100\% & 0.1046 \\
    $r=2, \gamma=0.3$  & 100\% & 0.1656 \\
    $r=2, \gamma=0.1$  & 100\% & 0.1848 \\
    $r=2, \gamma=0.03$  & 100\% & 0.1916 \\
    \bottomrule
  \end{tabular}
\end{table}

\begin{table}[H]
\centering
  \caption{Performance w.r.t. hyperparameters on 5-HNC.}
  \label{hyper5}
  \begin{tabular}{c c c}
    \toprule
    Hyperparameters & Recall & Coefficient error \\
    \midrule
    $r=10, \gamma=1$ & 100\% & 0.0270 \\
    $r=10, \gamma=0.3$  & 100\% & 0.0871 \\
    $r=10, \gamma=0.1$  & 100\% & 0.2169 \\
    $r=10, \gamma=0.03$  & 100\% & 0.3618 \\
    $r=5, \gamma=1$  & 100\% & 0.0304 \\
    $r=5, \gamma=0.3$  & 100\% & 0.0756 \\
    $r=5, \gamma=0.1$  & 100\% & 0.1030 \\
    $r=5, \gamma=0.03$  & 100\% & 0.1140 \\
    $r=2, \gamma=1$  & 100\% & 0.1156 \\
    $r=2, \gamma=0.3$  & 100\% & 0.1828 \\
    $r=2, \gamma=0.1$  & 100\% & 0.2039 \\
    $r=2, \gamma=0.03$  & 100\% & 0.2114 \\
    \bottomrule
  \end{tabular}
\end{table}

\section{Visualization}
\label{sec:more}
In this section, we visualize the estimated coefficients in HIN-PDE and compare them with the ground-truth coefficients. The visualized residual errors of the estimated coefficients show that the proposed our model is very accurate in coefficient estimation. 
The results of 1-HNC is already shown in Fig. \ref{est_coef}, so we only show the results of the other four cases of the governing equation of underground seepage here.
To be noted, the estimated coefficients of all datasets are all obtained with the optimal hyperparameters tuned for each dataset.
The visualizations on (2-5)-HNCs are shown in Figs. \ref{est4}-\ref{est7} respectively.
The conclusion stays the same as in Fig. \ref{est_coef}, which shows the effectiveness of our method.

\begin{figure}[H]
\centering
\includegraphics[width=0.8\linewidth]{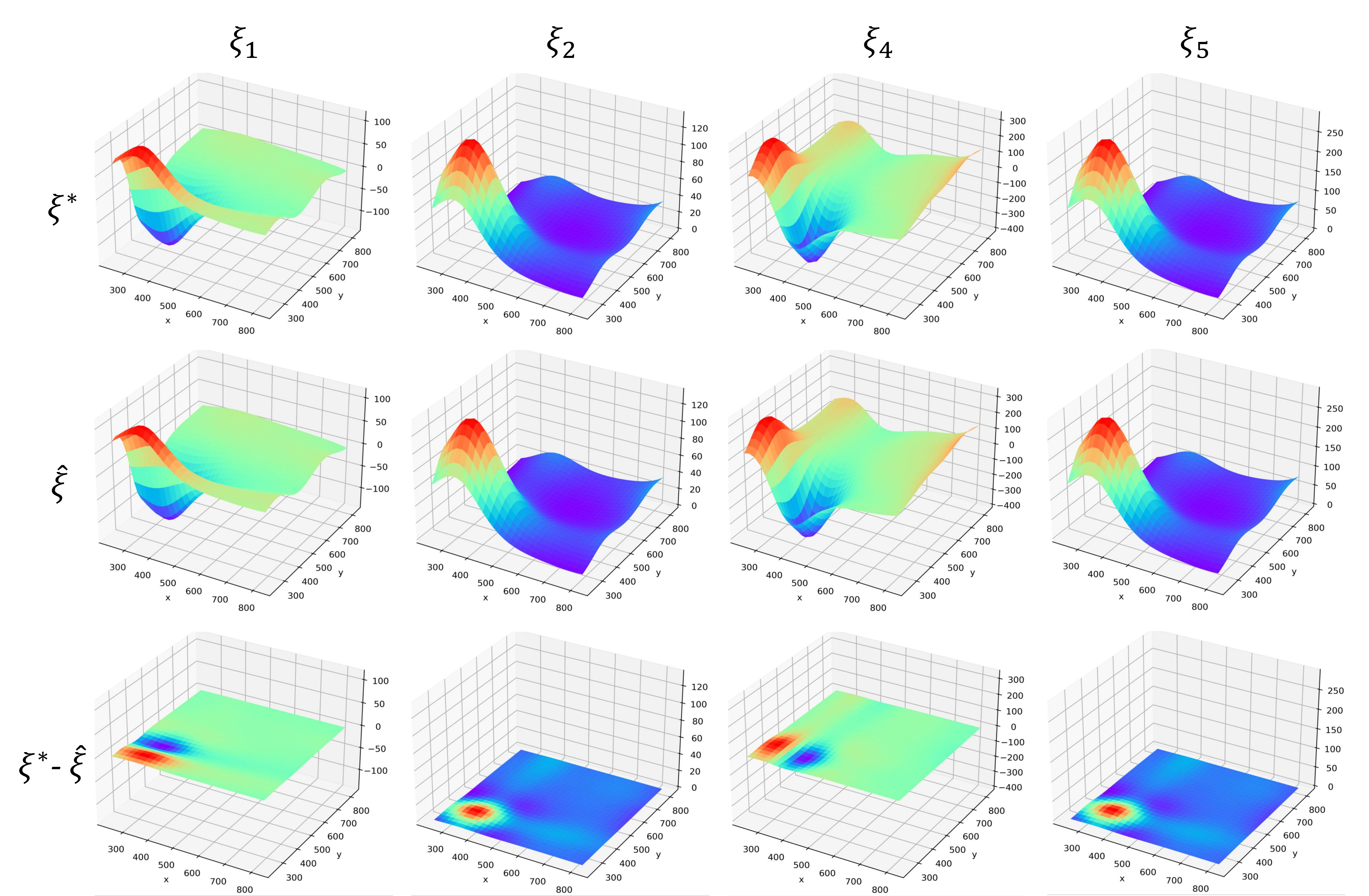}
\caption{Comparison of estimated and correct nonzero coefficients on 2-HNC.}
\label{est4}
\end{figure}

\begin{figure}[H]
\centering
\includegraphics[width=0.8\linewidth]{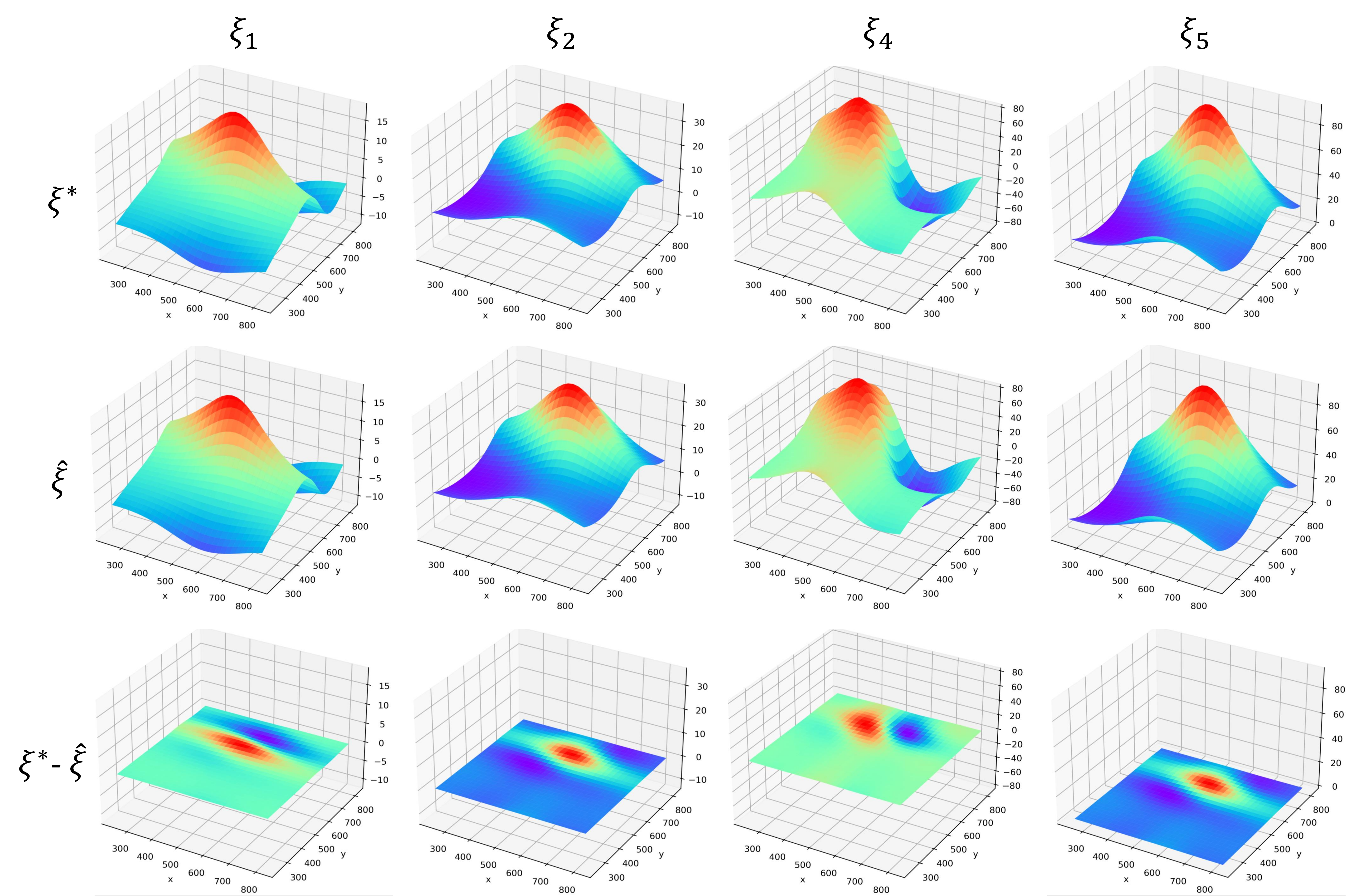}
\caption{Comparison of estimated and correct nonzero coefficients on 3-HNC.}
\label{est5}
\end{figure}

\begin{figure}[H]
\centering
\includegraphics[width=0.8\linewidth]{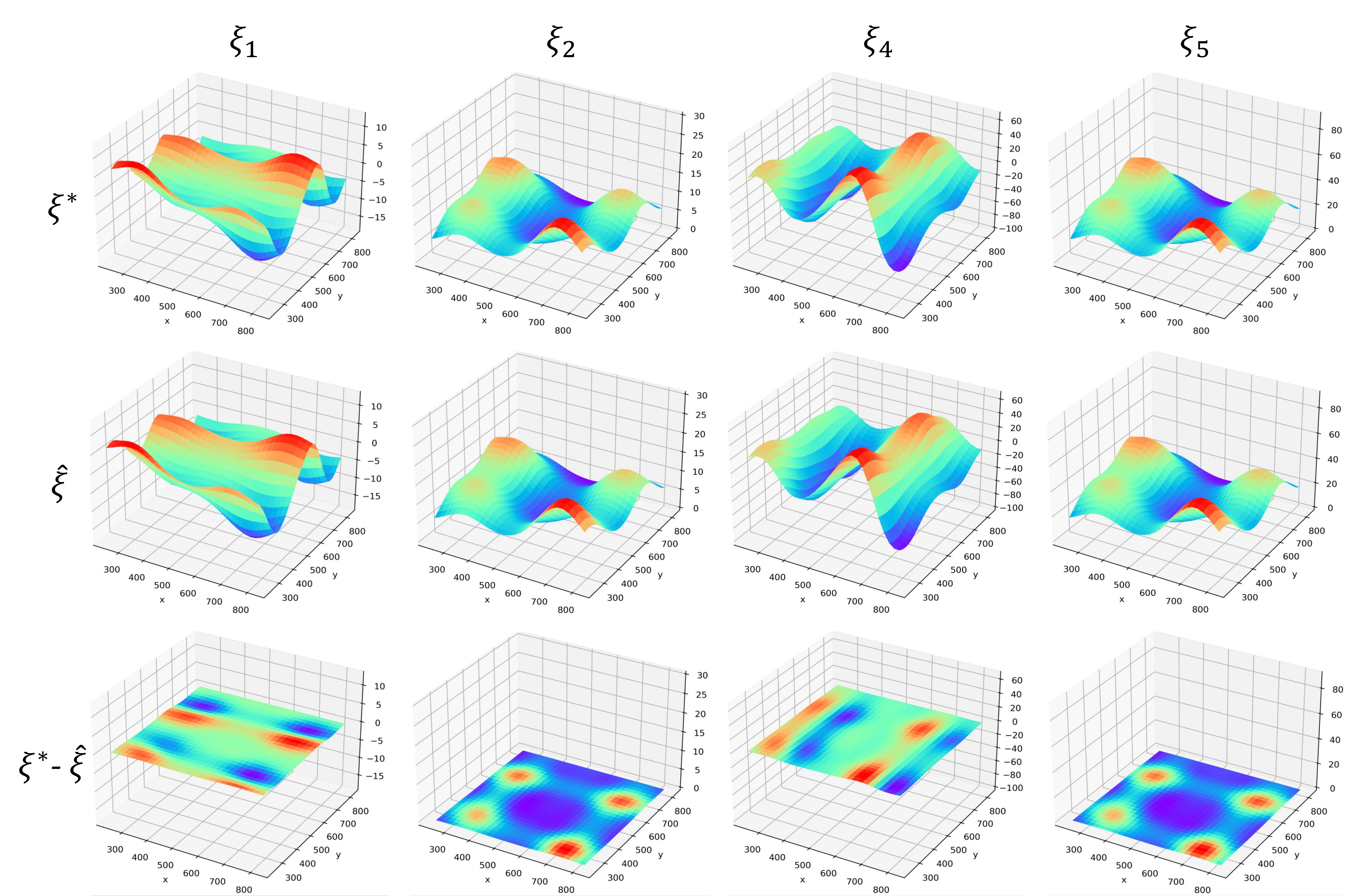}
\caption{Comparison of estimated and correct nonzero coefficients on 4-HNC.}
\label{est6}
\end{figure}

\begin{figure}[H]
\centering
\includegraphics[width=0.8\linewidth]{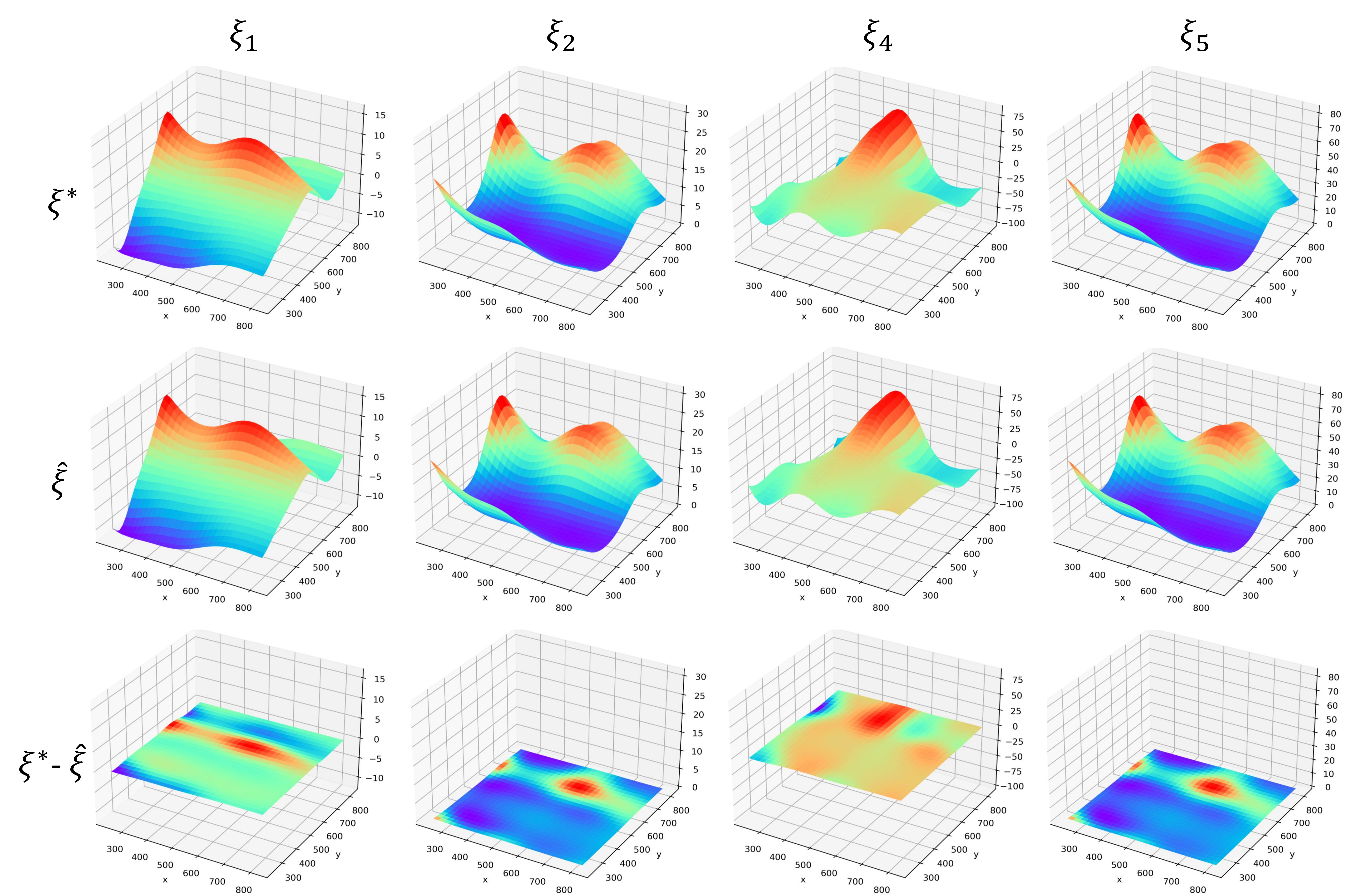}
\caption{Comparison of estimated and correct nonzero coefficients on 5-HNC.}
\label{est7}
\end{figure}

\end{document}